%% file: main.tex
\documentclass[sigconf]{aamas}

\usepackage{breqn}
\usepackage{balance} 

\usepackage{xcolor}
\usepackage{xspace}
\usepackage{algorithm}
\usepackage{amsmath}
\usepackage{natbib}
\usepackage{mathrsfs} 
\usepackage{algorithm}
\usepackage{subcaption}
\usepackage{algpseudocode}
\usepackage{array}
\usepackage{multirow}
\usepackage{hyperref}

\usepackage{comment}

\usepackage{amsthm}

\setcopyright{ifaamas}
\renewcommand\footnotetextcopyrightpermission[1]{}
\acmConference[AAMAS '26]{Proc.\@ of the 25th International Conference
on Autonomous Agents and Multiagent Systems (AAMAS 2026)}{May 25 -- 29, 2026}
{Paphos, Cyprus}{C.~Amato, L.~Dennis, V.~Mascardi, J.~Thangarajah (eds.)}
\copyrightyear{2026}
\acmYear{2026}
\acmDOI{}
\acmPrice{}
\acmISBN{}

\acmSubmissionID{615}

\title[AAMAS-2026 Formatting Instructions]{Conflict-Based Search for Multi Agent Path Finding with Asynchronous Actions}

\author{\href{https://orcid.org/0009-0000-5354-590X}{Xuemian Wu}}
\affiliation{
  \institution{Shanghai Jiao Tong University}
  \city{Shanghai}
  \country{China}}
\email{xuemian.wu@sjtu.edu.cn}

\author{\href{https://orcid.org/0000-0002-2351-6152}{Shizhe Zhao}}
\affiliation{
  \institution{Shanghai Jiao Tong University}
  \city{Shanghai}
  \country{China}}
\email{shizhe.zhao@sjtu.edu.cn}

\author{\href{https://orcid.org/0000-0003-2880-8653}{Zhongqiang Ren$^{\dagger}$}}
\affiliation{
  \institution{Shanghai Jiao Tong University}
  \city{Shanghai}
  \country{China}}
\email{zhongqiang.ren@sjtu.edu.cn}

\begin{abstract}

\input{abstract}
\end{abstract}

\keywords{Conflict-based Search; Multi Agent Path Finding; Asynchronous Actions}

\newcommand{\BibTeX}{\rm B\kern-.05em{\sc i\kern-.025em b}\kern-.08em\TeX}

\newtheorem{theorem}{Theorem}
\newtheorem{lemma}{Lemma}
\newtheorem{definition}{Definition}

\newtheorem{assumption}{Assumption}

\newtheorem{remark}{Remark}
\newtheorem{example}{Example}

\newcommand{\abbrMAPF}{MAPF\xspace}
\newcommand{\abbrMAPFAA}{MAPF-AA\xspace}

\newcommand{\abbrMAPFR}{\textnormal{MAPF\textsubscript{R}}\xspace}
\newcommand{\abbrCBS}{CBS\xspace}
\newcommand{\abbrCCBS}{CCBS\xspace}
\newcommand{\abbrCSIPP}{CSIPP\xspace}
\newcommand{\abbrCBSAA}{CBS-AA\xspace}
\newcommand{\abbrSMWTP}{CSA\xspace} 
\newcommand{\abbrMMWTI}{CMA\xspace} 
\newcommand{\abbrMMWTIS}{CMAS\xspace} 
\newcommand{\abbrSIPP}{SIPP\xspace}
\newcommand{\abbrSIPPS}{SIPPS\xspace}
\newcommand{\abbrSIPPSWC}{SIPPS-WC\xspace}
\newcommand{\abbrLSS}{LS-M*\xspace}
\newcommand{\movein}{\texttt{IN}\xspace}
\newcommand{\moveout}{\texttt{OUT}\xspace}
\newcommand{\trans}{\ensuremath{A}\xspace}
\newcommand{\intv}{\ensuremath{\tau}\xspace} 
\newcommand{\wait}{\texttt{WAIT}\xspace}
\newcommand{\inin}{\texttt{IN-IN}\xspace}
\newcommand{\outin}{\texttt{OUT-IN}\xspace}
\newcommand{\waitin}{\texttt{WAIT-IN}\xspace}
\newcommand{\cftij}{\ensuremath{\langle\trans^i,\trans^j,v\rangle}\xspace}
\newcommand{\Fv}[1]{\ensuremath{v_f(#1)}\xspace}
\newcommand{\Tv}[1]{\ensuremath{v_t(#1)}\xspace}
\newcommand{\Ft}[1]{\ensuremath{t_f(#1)}\xspace}
\newcommand{\Tt}[1]{\ensuremath{t_t(#1)}\xspace}
\newcommand\procName[1]{\textsl{#1}}
\newcommand{\durin}{\ensuremath{\tau_{in}}\xspace}
\newcommand{\durout}{\ensuremath{\tau_{out}}\xspace}

\newcommand{\cstrm}[4]{\ensuremath{\langle{#1},#2\rightarrow #3,#4\rangle_{m}}\xspace}
\newcommand{\cstrw}[3]{\ensuremath{\langle{#1},#2, #3 \rangle_{w}}\xspace}
\newcommand{\cstro}[3]{\ensuremath{\langle{#1},#2, #3 \rangle_{o}}\xspace}

\begin{document}


\pagestyle{fancy}
\fancyhead{}



\maketitle 


\section{Introduction}\label{sec:intro}
\input{introduction}

\section{Problem Formulation}\label{sec:problem}

\input{problem_formulation}

\section{Preliminaries}\label{sec:preliminaries}
\input{preliminaries}

\section{Method}\label{sec:method}

\input{method}

\input{analysis_v2}

\section{Experimental Results}\label{sec:expr}
\input{result}

\section{Conclusion and Future Work}\label{sec:conclude}
\input{conclude}

\begin{acks}
This work was supported by the Natural Science Foundation of China under Grant 62403313, and the Natural Science Foundation of Shanghai under Grant 24ZR1435900.
\end{acks}



\bibliographystyle{ACM-Reference-Format} 
\bibliography{./sample}


\end{document}

%% file: abstract.tex
Multi-Agent Path Finding (MAPF) seeks collision-free paths for multiple agents from their respective start locations to their respective goal locations while minimizing path costs.
Most existing MAPF algorithms rely on a common assumption of synchronized actions, where the actions of all agents start at the same time and always take a time unit, which may limit the use of MAPF planners in practice.
To get rid of this assumption, Continuous-time Conflict-Based Search (CCBS) is a popular approach that can find optimal solutions for MAPF with asynchronous actions (MAPF-AA).
However, CCBS has recently been identified to be incomplete due to an uncountably infinite state space created by continuous wait durations.
This paper proposes a new method, Conflict-Based Search with Asynchronous Actions (CBS-AA), which bypasses this theoretical issue and can solve MAPF-AA with completeness and solution optimality guarantees.
Based on CBS-AA, we also develop conflict resolution techniques to improve the scalability of CBS-AA further.
Our test results show that our method can reduce the number of branches by up to 90\%.

%% file: introduction.tex
Multi-Agent Path Finding (MAPF) seeks collision-free paths for multiple agents from their respective start locations to their respective goal locations while minimizing path costs.
The environment is often represented by a graph, where vertices represent the locations that the agent can reach, and edges represent actions that transit the agent between two locations.
MAPF is NP-hard to solve optimally~\cite{yu2013structure}, and a variety of MAPF planners were developed, ranging from optimal planners~\cite{sharon2015conflict,wagner2015subdimensional}, bounded sub-optimal planners~\cite{barer2014suboptimal,li2021eecbs}, to unbounded sub-optimal planners~\cite{okumura2022priority,de2013push}.
A common underlying assumption in these planners is that each action of an agent, either waiting in place or moving to an adjacent vertex, takes the same duration, i.e., a time unit, and the actions of all agents are synchronized, i.e., the action of each agent starts at the same discrete time step.
This assumption limits the application of MAPF planners, especially when the agent speeds are different or an agent has to vary its speed when going through different edges (Fig.~\ref{fig:f1}).

\begin{figure}[tb]
    \centering
    \includegraphics[width=1\linewidth]{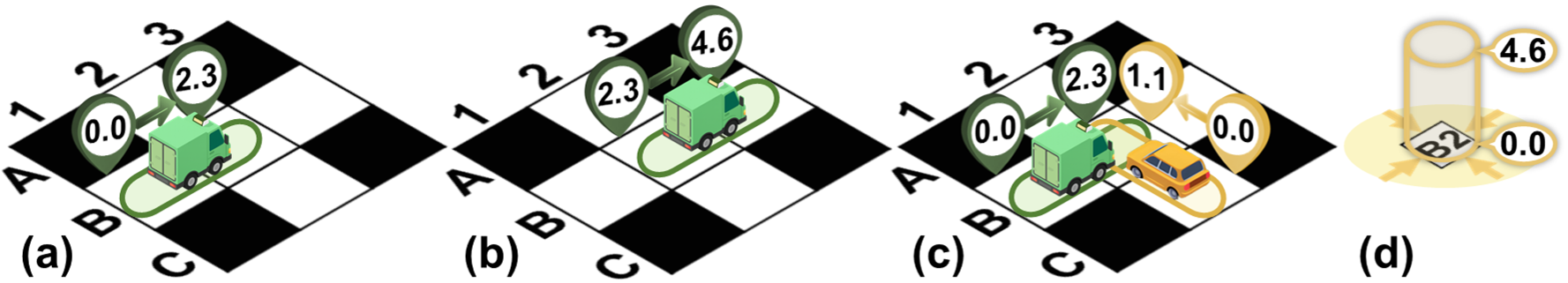}
    \caption{A motivating example of \abbrMAPFAA where the yellow car moves fast and the green truck moves slowly in continuous time. The circled numbers show the time points: e.g., in (a), the truck moves from B1 to B2 during the time range $[0.0,2.3]$. This work considers the agent to occupy both ends of an edge when the agent goes through it. As a result, a constraint (as shown in (d)) at B2 with time range $[0.0,4.6]$ is imposed on the yellow car to avoid collision as shown in (c).
    }
    \label{fig:f1}
\end{figure}

To bypass this synchronous action assumption, MAPF variants such as Continuous-Time MAPF~\cite{ANDREYCHUK2022103662}, MAPF with Asynchronous Actions (MAPF-AA)~\cite{ren2021loosely}, MAPF$_R$~\cite{walker2018extended} were proposed.
The major idea in those variants is that, the actions of agents can take different amounts of time, and as a result, the agents may not start and end each of their actions at the same discrete time steps.
Among the exact algorithms that can find optimal solutions, Continuous-time Conflict-Based Search (CCBS)~\cite{ANDREYCHUK2022103662} is a leading approach that extends \abbrCBS to handle various action durations.

This paper focuses on exact algorithms for MAPF-AA, which considers duration conflict~\cite{ren2021loosely,okumura2021time} in the sense that when an agent traverses an edge within a time range, the agent occupies both end vertices of the edge during that time range and no two agents can occupy the same vertex at the same time.
CCBS~\cite{ANDREYCHUK2022103662} can be applied to solve MAPF-AA. However, a naive application can lead to incompleteness due to the infinite number of possible durations for the wait action~\cite{li2025cbs}, and may not be able to find a solution even if the instance is solvable.
Recent works~\cite{li2025cbs,combrink2025sound} have analyzed and discussed the completeness issue of CCBS. 
This paper provides a possible fix by developing a new exact algorithm called Conflict-Based Search with Asynchronous Actions (CBS-AA) for MAPF-AA.
Besides CCBS~\cite{ANDREYCHUK2022103662}, our prior work has studied scalable yet unbounded sub-optimal algorithms for \abbrMAPFAA~\cite{zhou2025loosely,2025_SOCS_LSRPstar_ShuaiZhou}.
In terms of exact algorithms, our prior work developed an A*/M*-based exact algorithm for \abbrMAPFAA called Loosely Synchronized M* (\abbrLSS)~\cite{ren2021loosely}, which often runs slower than CCBS~\cite{ANDREYCHUK2022103662}.

Additionally, in MAPF-AA, it takes a different amount of time when different agents go through the same edge, or the same agent goes through different edges.
Such heterogeneous duration tends to complicate the collision avoidance among the agents, and disables many conflict resolution techniques in CBS for MAPF, which limits the scalability of the exact algorithms for MAPF-AA.
To improve scalability, we develop constraint propagation techniques based on the agents' action duration within CBS-AA.
The intuition behind these techniques is that when two agents collide, we add constraints to the agents to forbid as many actions as possible, and forbid each action as long as possible, so that CBS-AA can resolve collisions in fewer iterations.

We compare different variants of our \abbrCBSAA with the existing \abbrCCBS.
The results show that with the proposed constraint propagation techniques, the success rate of CBS-AA is significantly higher than that of \abbrCCBS~\cite{ANDREYCHUK2022103662}, and the number of iterations in high-level is reduced by up to 90\%.
We also compare our CBS-AA with \abbrLSS~\cite{ren2021loosely}. 
The results show that \abbrCBSAA can find the optimal solution faster, and the costs of the optimal solutions are the same as those of \abbrLSS.

%% file: problem_formulation.tex
Let $I = \{1, 2, \dots, N\}$ be the index set of $N$ agents, where each $i \in I$ corresponds to a specific agent.
The workspace is represented by an undirected graph $G=(V,E)$, where $V$ is the set of traversable vertices, and $E\subset V\times V$.
Each edge $e=(u,v) \in E$ represents an action that moves an agent from $u$ to an adjacent vertex $v$.
The travel time of each edge may vary for each agent.
Let $\tau^i(u,v) \in \mathbb{R}_{\ge 0}$ denote the travel time for agent $i$ moving from $u$ to $v$.
All agents can wait at a vertex for an arbitrary amount of time.

Let $s^i=(v, t)$ denote the space-time state of agent $i$, 
and let $\trans^i=(s^i_1, s^i_2)$ denote a state transition from $s^i_1$ to $s^i_2$.
Given $\trans^i=((v^i_1,t^i_1), (v^i_2, t^i_2))$, 
let $\Fv{\trans^i}, \Ft{\trans^i}$ and $\Tv{\trans^i}, \Tt{\trans^i}$ denote the space and time components for $s^i_1$ and $s^i_2$, respectively, and let $\intv(\trans^i)$ denote the duration of $A^i$.
For any action, $t^i_2 = t^i_1 + \intv(\trans^i)$.
For a \emph{move} action, $v^i_2$ is adjacent to $v^i_1$, and $\intv(\trans^i) = \tau^i(v^i_1,v^i_2)$ is given by input.
For a \emph{wait} action, we have $v^i_1 = v^i_2$, and the duration $\intv(\trans^i) \in \mathbb{R}_{\ge 0}$ is determined by the planner.
We use the same conflict model from existing work~\cite{ren2021loosely,okumura2021time}, as illustrated below.
\begin{definition}[Duration Occupancy]\label{def:DO}
When agent $i$ performs an action $((v^i_1,t^i_1),(v^i_2,t^i_2))$, both $v^i_1$ and $v^i_2$ are occupied by $i$ during the action, 
which is called Duration Occupancy (DO).
Specifically, $v^i_1$ is occupied at time $t^i_1$, $v^i_2$ is occupied at time $t^i_2$,
and both $v^i_1,v^i_2$ are occupied during $(t^i_1, t^i_2)$.
At any time point, a vertex can only be occupied by at most one agent.
Multiple agents are in conflict if they both occupy the same vertex for a non-empty time interval, 
which is referred to as Duration Conflict (DC).
\end{definition}



Let $\pi^i(v_s, v_g)=(a^i_0=(v_s,0), a^i_1,\dots, a^i_k=(v_g,t_k))$ denote a path from 
$v_s$ to $v_g$. The cost of $\pi^i$ is $t_k$, denoted as $g(\pi^i)$, 
which is the time point it reaches $v_g$ and can permanently stay after $t_k$ without conflict.
Let $V_s, V_g$ denote the set of start and goal locations of all agents, and $v_s^i \in V_s, v_g^i \in V_g$ denote the start and goal location of agent $i$ respectively.
We assume there is no conflict when all agents stay at their start or goal locations.
Multi Agent Path Finding with Asynchronous Action (\abbrMAPFAA) $P = \langle V, E, V_s, V_g \rangle$ seeks to find a set of conflict-free paths
such that (1) each agent $i\in I$ starts at $v_s^i$ and ends at $v_g^i$; and (2) minimize the sum of costs (SoC) of all agents' paths, i.e., $\min \sum_{i\in I}g(\pi^i)$.



%% file: preliminaries.tex
\paragraph{CBS}
Conflict-Based Search (\abbrCBS)~\cite{sharon2015conflict} is a two-level search algorithm that finds an optimal joint path for \abbrMAPF.
At the high-level, \abbrCBS constructs a search tree and 
starts with a root node which consists of all agents’ individually optimal path ignoring any conflict.
And then \abbrCBS selects a specific node that has the smallest $g$-value (sum of costs), and detects conflicts along the paths of any pair of agents.
According to the detected conflict, two constraints are generated corresponding to the two sub-trees.
For each of those two constraints, \abbrCBS runs a low-level search to find a new path satisfying the constraints.
At the low-level, a single-agent planner is invoked to plan an optimal path that satisfies all constraints related to a specific agent.
\abbrCBS guarantees finding a conflict-free joint path with the minimal sum of costs.

\paragraph{\abbrCCBS}
Continue Conflict-Based Search (\abbrCCBS)~\cite{ANDREYCHUK2022103662} is a \abbrCBS-based method and can solve the \abbrMAPFR problem to optimality. 
It assumes the travel time for an edge is real-valued and non-uniform, and different edges have different travel times.
Therefore, the time is continue and the action of agents is asynchronous.
To detect conflicts, \abbrCCBS assumes that all agents move in a straight line at a constant speed and detects collisions based on agents' geometry, such as circles.
To resolve conflicts, \abbrCCBS adds constraints over pairs of actions and time ranges, instead of location-time pairs.

Specifically, for a conflict $\langle(a^i,t^i),(a^j,t^j)\rangle$, which means that agent $i$ performs action $a^i$ at $t^i$ and agent $j$ performs action $a^j$ at $t^j$, and they collide, \abbrCCBS computes for each action an unsafe intervals w.r.t the other's action.
The unsafe interval $[t^i,t_u^i)$ of $(a^i,t^i)$ w.r.t. $(a^j,t^j)$ is the maximal time interval starting from $t^i$ in which if agent $i$ performs $a^i$ then it is in conflict with the action $(a^j,t^j)$.
\abbrCCBS adds to agent $i$ the constraint $\langle i,a^i,[t^i,t_u^i)\rangle$, which means agent $i$ cannot perform action $a^i$ in the range $[t^i, t_u^i)$), and adds to agent $j$ the constraint $\langle j,a^j,[t^j,t_u^j)\rangle$.

The low-level solver of \abbrCCBS is Constrained Safe Interval Path Planning (\abbrCSIPP), where safe intervals for a vertex are computed based on \abbrCCBS constraints.
In detail, a \abbrCCBS constraint $\langle i,a^i_w,[t^i,t_u^i)\rangle$ about a wait action $a^i_w$ of vertex $v$ will divide the safe interval of $v$ to two parts: one that ends at $t^i$ and another that starts at $t_u^i$.
If the action $a^i_m$ is a move action related to edge $e = (v,v')$, \abbrCCBS replaces the action $a^i_m$ with action $a^{i}_{m'}$ that starts by waiting in $v$ for duration $t_u^i-t^i$ before moving to $v^{'}$.

\abbrCCBS has issues as reported in~\cite{li2025cbs}.
\abbrCCBS fails to resolve conflicts for wait actions, since the time is continue, there are infinitely many wait actions with duration time $\tau_w \in \mathbb{R}_{\ge 0}$.
The constraint added by branching each time in \abbrCCBS only focuses on a specific duration of the wait action, which can lead to infinite number of branching when resolving conflict caused by wait actions.
For example, wait action $a_w^i = (B,B,2.0)$ means waiting at vertex $B$ for a duration of $2.0$, and \abbrCCBS adds a constraint $\langle i,a_w^i,[t^i,t_u^i)\rangle$ to prohibit agent $i$ from performing $a_w^i$ at time $t \in [t^i,t_u^i) $. 
But agent $i$ still can preform wait action $a_{w1}^i = (B,B,2.01), a_{w2}^i = (B,B,2.001), a_{w3}^i = (B,B,2.0001)...$ at time $t \in [t^i,t_u^i) $.
So, \abbrCCBS may not terminate when there is a wait action.

In the open-sourced implementation\footnote{\url{https://github.com/PathPlanning/Continuous-CBS}}, \abbrCCBS makes a change when transferring constraints about wait action to low-level solver.
For the previous constraint $\langle i,a^i_w,[t^i,t_u^i)\rangle$, \abbrCSIPP divides the safe interval of $B$ to two parts: $[0,t^i)$ and $[t_u^i,\infty)$, which means that agent $i$ can not perform any wait action $a = (B,B,\tau_w), \tau_w \in \mathbb{R}_{\ge 0}$ in $[t^i,t_u^i)$.
However, ``not perform any wait action $a = (B,B,\tau_w), \tau_w \in \mathbb{R}_{\ge 0}$ in $[t^i,t_u^i)$'' is not equivalent to ``not perform wait action $a^i_w = (B,B,2.0)$ in $[t^i,t_u^i)$''. 
This inconsistency may miss feasible solutions during branching and can not guarantee completeness and optimality.
For the rest of the paper, we refer to this implemented version as \abbrCCBS.


\begin{figure}[tb]
    \centering
    \includegraphics[width=1\linewidth]{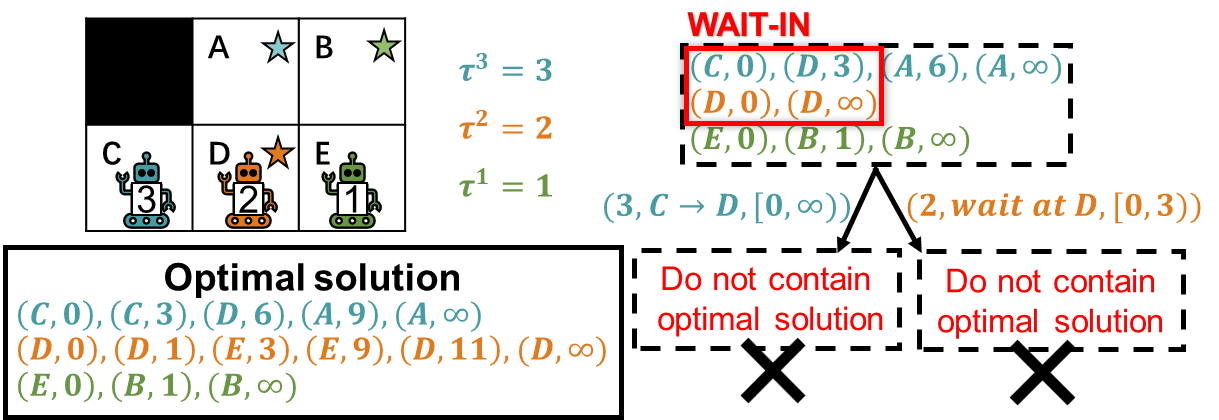}
    \caption{Toy example for the issues of \abbrCCBS. }
    \label{fig:toyexample_ccbs}
\end{figure}

\begin{example}\label{ex:toyexample_ccbs}
As shown in Fig.~\ref{fig:toyexample_ccbs}, there are three agents $I=\{1,2,3\}$.
The duration of all move actions of agent $1$, $2$ and $3$ is $\tau^1 = 1$, $\tau^2 = 2$ and $\tau^3 = 3$, respectively.
In the root node of high-level, agent $2$ and $3$ have a conflict $\langle(a^2,t^2=0),(a^3,t^3=0)\rangle$ where $a^2 = (D,D,\infty)$ is a wait action with $\infty$ duration time, $a^3 = (C,D,3)$ is a move action from $C$ to $D$.
By Def.~\ref{def:DO}, the unsafe interval of $(a^2,t^2=0)$ w.r.t. $(a^3,t^3=0)$ is $[0, 3)$ and the unsafe interval of $(a^3,t^3=0)$ w.r.t. $(a^2,t^2=0)$ is $[0, \infty)$. 
By the implementation of \abbrCCBS, two constraints are generated, $(3,C \rightarrow D, [0, \infty)])$ forbids agent $3$ to perform action $a^3$ in the range $[0, \infty)$ and $(2,wait\;at\;D, [0, 3))$ forbids agent $2$ to perform any wait action $a = (D,D,\tau_w), \tau_w \in \mathbb{R}_{\ge 0}$ in the range $[0, 3)$. 
The sub-trees of both branches do not contain the optimal solution. 
\end{example}

\paragraph{\abbrLSS}
Loosely Synchronized M* (\abbrLSS)~\cite{ren2021loosely} solves \abbrMAPFAA by introducing new search states that include both the locations and the action times of the agents.
Similar to A*, LSS iteratively selects states from an open list, expands them to generate successors, prunes those that are either conflicting or less promising, and inserts the remaining ones into the open list for future expansion.
This process continues until a conflict-free joint path from the start locations to the goal locations is found for all agents.
\abbrLSS further introduce the idea of subdimensional expansion~\cite{wagner2015subdimensional} into LSS and can handle more agents than LSS.
\abbrLSS is complete and finds an optimal solution for \abbrMAPFAA, but can only handle a relatively small number of agents.

%% file: method.tex
This section proposes Conflict Based Search with Asynchronous Action (\abbrCBSAA), which finds an optimal solution for \abbrMAPFAA. 
We first modify \abbrCCBS to effectively resolve conflicts and call this modified method Constraint on Single Action (\abbrSMWTP).
Then, we use DO to propagate constraints and resolve conflicts efficiently, which we call Constraint on Multiple Actions (\abbrMMWTI).


\paragraph{Overview} \abbrCBSAA (Alg.~\ref{alg:CBSAA}) is similar to \abbrCBS with three processes modified: \procName{LowLevelPlan}, \procName{DetectConflict} and \procName{GenerateConstraints}.
\procName{LowLevelPlan} adapts Safe Interval Path Planning (\abbrSIPP) ~\cite{phillips2011sipp} to handle continuous-time.
The numbers associated with safe intervals and constraints are all positive real numbers.
We cut safe intervals into continuous time intervals according to the constraints, rather than a set of discrete time steps.
\procName{DetectConflict} detects conflicts in the continuous time range. 
When there is overlap in the time intervals for two agents to occupy a same vertex, a conflict is returned.
In \procName{GenerateConstraints}, we propose two different conflict resolution methods for \abbrMAPFAA, \abbrSMWTP and \abbrMMWTI, as detailed later.
\input{algs/cbs}


\input{method_conflict}

%% file: algs/cbs.tex
\begin{algorithm}[tbp]
    \small
	\caption{\abbrCBSAA}\label{alg:CBSAA}
	\begin{algorithmic}[1]
        \Statex{INPUT: $G=(V,E)$}
        \Statex{OUTPUT: a conflict-free joint path $\pi$ in $G$.}
		\State{$\Omega_c \gets \emptyset$, $\pi_{},g_{} \gets$ \procName{LowLevelPlan}($ \Omega_c$)}
		\State{Add $P_{root,1}=(\pi,g,\Omega_c)$ to OPEN}\label{cbssapx:alg:cbssa:lineInitEnd}
		\While{$\text{OPEN} \neq \emptyset$} 
		\State{$P=(\pi,g,\Omega_c) \gets$ OPEN.pop()}\label{cbssapx:alg:cbssa:linePopOpen}
		\State{$cft \gets$ \procName{DetectConflict}($\pi$)}\label{cbssapx:alg:cbssa:lineDetectCft}
		\State{\textbf{if} $cft = NULL$ \textbf{then} \textbf{return} $\pi$ }
		\State{$\Omega \gets $ \procName{GenerateConstraints}($cft$)}\label{cbssapx:alg:cbssa:lineGenCstr}
		\ForAll{$\omega^i \in \Omega$}
		\State{$\Omega' = \Omega_c \cup \{\omega^i\}$}
            \State{$\pi',g' \gets$ \procName{LowLevelPlan}\label{cbssapx:alg:cbssa:lineLowLevelReplan}($\Omega'$)}
		\State{Add $P'=(\pi',g',\Omega')$ to OPEN}\label{cbssapx:alg:cbssa:lineAddOpen}
		\EndFor
		\EndWhile \label{}
		\State{\textbf{return} failure}
	\end{algorithmic}
\end{algorithm}

%% file: method_conflict.tex
\subsection{Conflict Detection and Classification}\label{sec:detect}
For a vertex $v$, there are three types of actions:
\begin{equation}
\begin{cases}
	\movein:&\{\trans^i | \Tv{\trans^i}=v\} \\
	\moveout:&\{\trans^i | \Fv{\trans^i}=v\}\\
	\wait:&\{\trans^i | \Fv{\trans^i}=\Tv{\trans^i}=v\}
\end{cases}
\end{equation}
If agent $i$ wants to go through $v$, it must perform these three actions $\trans^i_I \in \movein$, $\trans^i_W \in \wait$ and $\trans^i_O \in \moveout $ at $v$ in sequence.
Let $\intv(\trans^i, v)$ denote the time interval during which the transition of $i$ occupies vertex $v$ based on Def.~\ref{def:DO}. 
If agent $i$ performs $\trans^i_I$ at $t$, $\intv(\trans^i_I, v) = (t,t+\tau(\trans^i_I)]$, $\intv(\trans^i_W, v) = [t+\tau(\trans^i_I),t+\tau(\trans^i_I)+\tau(\trans^i_W)]$ and $\intv(\trans^i_O, v) = [t+\tau(\trans^i_I)+\tau(\trans^i_W),t+\tau(\trans^i_I)+\tau(\trans^i_W)+\tau(\trans^i_O)]$.
If $i$ does not need to wait at $v$, the duration $\tau(\trans^i_W)$ is $0$ and the wait action $\trans^i_W$ occupies $v$ only at one time point $t+\tau(\trans^i_I)$.

Two agent $i$ and $j$ are in conflict if there is a $v$ such that $\intv(\trans^i, v) \cap \intv(\trans^j, v) \neq \emptyset$. 
Let $\cftij$ denote a duration conflict between agents $i$ and $j$ that both occupy the same vertex $v$ during their actions $\trans^i, \trans^j$.
For two agent $i$ and $j$, we always detect and resolve the earliest conflict between them.
We classify all conflicts to be resolved into three types (Fig.~\ref{fig:conflict-type}):
\begin{itemize}
	\item \inin: $\cftij_{I}$, where $v=\Tv{\trans^i}=\Tv{\trans^j}$;
	\item \outin: $\cftij_{O}$, where $v=\Tv{\trans^i}=\Fv{\trans^j}$;
	\item \waitin: $\cftij_{W}$, where $v=\Tv{\trans^i}=\Fv{\trans^j}=\Tv{\trans^j}$.
\end{itemize}
While there are nine possible combinations of two agents' actions $\{\text{\movein,\moveout,\wait{}}\}\times\{\text{\movein,\moveout,\wait{}}\}$, we only need to consider the combinations that involve \movein.
Since for any other conflicts $\{\text{\moveout,\wait{}}\}\times\{\text{\moveout,\wait{}}\}$,
an \movein must be involved prior to \wait or \moveout and cause a conflict at the same vertex due to DO.
From now on, for any conflict between $i$ and $j$, let $i$ be the agent who performs the \movein action, and $j$ may or may not perform \movein.

\begin{figure}[tb]
    \centering
    \includegraphics[width=0.8\linewidth]{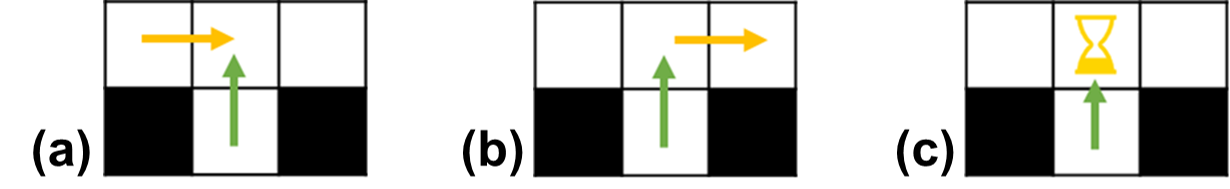}
    \caption{Three Conflict Types. (a): \inin; (b): \outin; (c): \waitin}
    \label{fig:conflict-type}
\end{figure}

\subsection{Constraint and Low-level Planner}\label{sec:cstr-sipp}
Safe Interval Path Planning (\abbrSIPP) ~\cite{phillips2011sipp} is often used as the low-level planner in CBS. 
It constructs a search space with states defined by their vertex and safe time interval, resulting in a graph that generally only has a few states per vertex.
\abbrSIPP is more efficient than A* in the presence of wait durations and finds an optimal path that avoids any unsafe time intervals.
We adapt \abbrSIPP to \abbrMAPFAA setting by using the following constraints for an agent $i$,
let $A^i=((v^i_1, t^i_1), (v^i_2, t^i_2))$:
\begin{itemize}
	\item Motion Constraint (MC) $\cstrm{i}{u}{v}{[l, r)}$: 
		forbids all move actions $\trans^i$ where $v^i_1=u,v^i_2=v$ and $t^i_1\in[l,r)$;
	\item Wait Constraint (WC) $\cstrw{i}{v}{[l, r)}$:
		forbids all wait actions $\trans^i$ where $v^i_1=v$ and $\intv(\trans^i, v)\cap [l,r)\neq\emptyset$
	\item Occupancy Constraint (OC) $\cstro{i}{v}{t}$:
		forbids all actions (\movein, \moveout and \wait) $A^i$ where $t\in\intv(\trans^i,v)$;
\end{itemize}
Fig.~\ref{fig:lowlevelsipp} illustrates how the constraints affect the search space of the low-level planner. 


\begin{figure}[tb]
    \centering
    \includegraphics[width=0.9\linewidth]{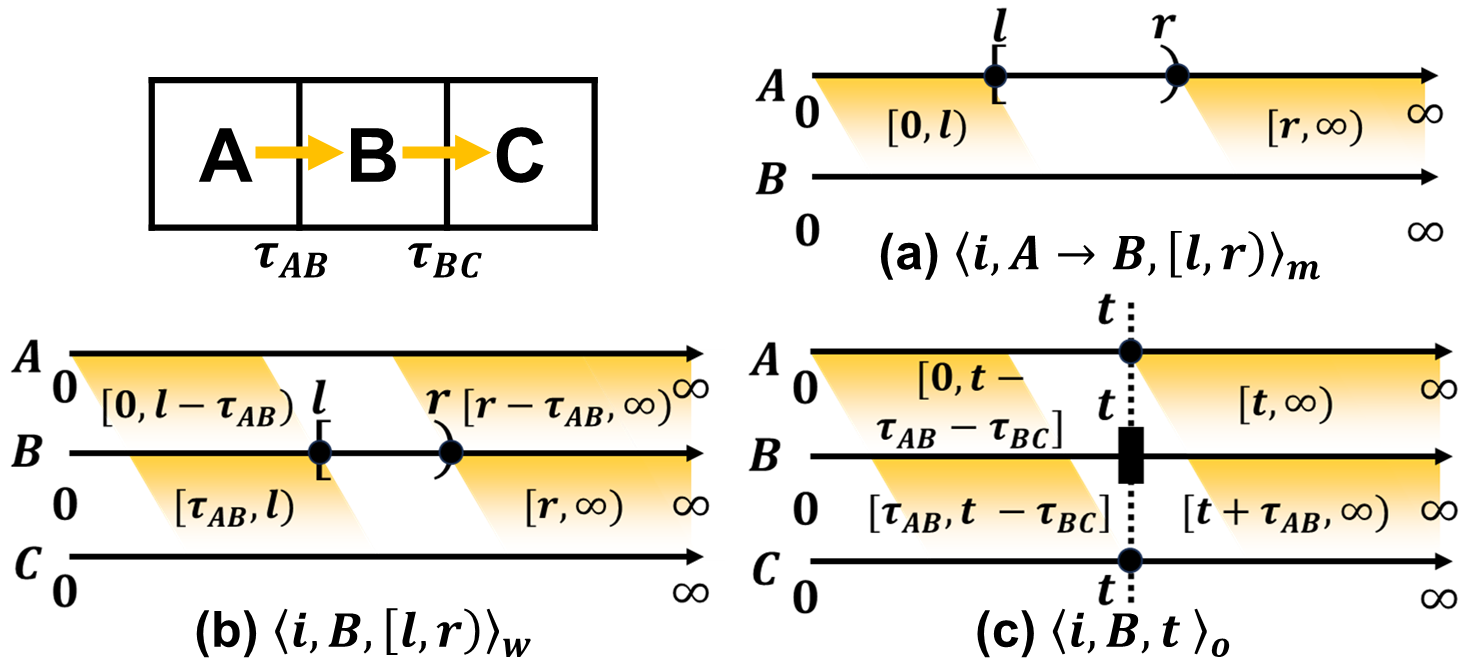}
    \caption{Changes in the search space of low-level after adding constraints. Duration time from $A$ to $B$ and from $B$ to $C$ are abbreviated as $\tau_{AB}$ and  $\tau_{BC}$.
    (a) MC $\cstrm{i}{A}{B}{[l, r)}$: before time $l$, the interval at which moving from $A$ to $B$ can be started is $[0,l)$; after time $r$, the interval is $[r,\infty)$. 
    (b) WC $\cstrw{i}{B}{[l, r)}$: before time $l$, the interval at which an \movein in $B$ can be started is $[0,l-\tau_{AB})$ and the interval at which an \moveout in $B$ can be started is $[\tau_{AB},l)$; after time $r$, the interval about \movein in $B$ is $[r-\tau_{AB},\infty)$ and the interval about \moveout in $B$ is $[r,\infty)$; the safe interval of $B$ is $[0,l)$ and $[r,\infty)$.
    (c) OC $\cstro{i}{B}{t}$: before time $t$, the interval at which an \movein in $B$ can be started is $[0,t-\tau_{AB}-\tau_{BC}]$ (by Def.~\ref{def:DO}, $t-\tau_{AB}-\tau_{BC}$ is included) and the interval at which an \moveout in $B$ can be started is $[\tau_{AB},t-\tau_{BC}]$; after time $t$, the interval about \movein in $B$ is $[t,\infty)$ and the interval about \moveout in $B$ is $[t+\tau_{AB},\infty)$; the safe interval of $B$ is $[0,t-\tau_{BC})$ and $[t+\tau_{AB},\infty)$. 
    }
    \label{fig:lowlevelsipp}
\end{figure}

In the low-level planner of \abbrCBS~\cite{sharon2015conflict}, tie-breaking is a useful method to find conflict-free solutions faster.
It can find an optimal path for agent $i$ that satisfies the constraints added by high-level and has fewer conflicts with the planned paths of other agents.
\abbrSIPPS \cite{li2022mapf} extends \abbrSIPP to consider other paths as soft constraints and breaks ties by preferring the path that has fewer soft conflicts (i.e., conflicts with soft constraints).
But \abbrSIPPS ignores the cases where an agent may encounter multiple soft conflicts if it waits within a safe interval.

To consider these cases, we propose \abbrSIPPS with Waiting Conflict (\abbrSIPPSWC) to consider the soft conflicts when waiting.
Specifically, a state $s=(v,t,t_h,c_v^w)$ in \abbrSIPPSWC consists of a vertex $v$, an arrival time $t$, an end time of the corresponding safe interval $t_h$, and an integer number $c_v^w$ indicating the number of soft conflicts if waiting at $v$ from $t$ to $t_h$.
As shown in Fig. \ref{fig:sipps_aa}, $c_v^w$ can help distinguish between a path that moves from $v$ to $v'$ and then waits at $v'$ and another path that waits at $v$ and then moves to $v'$.
To prune states, if two states $s_1$ and $s_2$ have the same $v$, $t_h$ and $c_v^w$, then $s_1$ dominants $s_2$ if the arrival time $s_1.t \leq s_2.t$ and the number of soft conflicts along the path from the start vertex to $v$ $c(s_1) \leq c(s_2)$.
Our low-level planner adapts \abbrSIPPSWC to continuous time and the three types of constraints as aforementioned (Fig.~\ref{fig:sipps_aa}).

\begin{figure}[tb]
    \centering
    \includegraphics[width=0.8\linewidth]{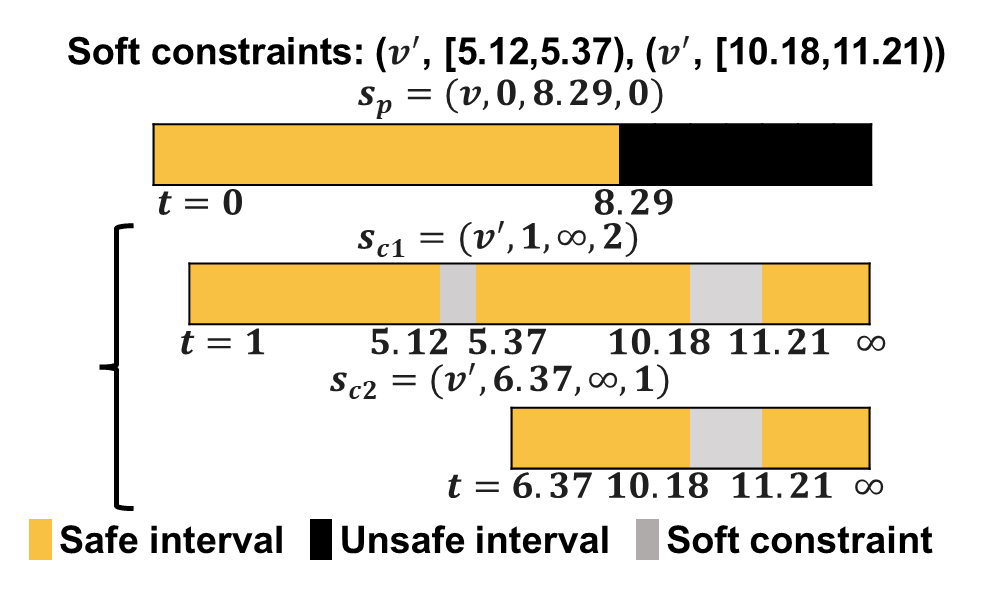}
    \caption{ Expanding states in \abbrSIPPSWC. The parent state $s_p$ at vertex $v$ with safe interval $[0,8.29)$ can get two child states $s_{c1}$ and $s_{c2}$ at vertex $v'$. State $s_{c1}$ has safe interval $[1,\infty)$ (starts moving at $t=0$, arrives at $v'$ at $t=1$) and $s_{c1}.c^w_v = 2$. State $s_{c2}$ has safe interval $[6.37,\infty)$ (waits at $v$, starts moving at $t=5.37$, arrives at $v'$ at $t=6.37$) and $s_{c2}.c^w_v$ = 1}
    \label{fig:sipps_aa}
\end{figure}

\subsection{Constraints on Single Action (\abbrSMWTP)}\label{sec:smwtp}
Let $\trans^i=((v^i_1,t^i_1),(v^i_2,t^i_2))$ and $\trans^j=((v^j_1, t^j_1),(v^j_2,t^j_2))$.
In an \inin conflict $\cftij_I$, if we permit $j$'s current action,
then $i$ cannot perform its action during the time interval $[l^i=t^i_1, r^i=t^j_2)$.
Here, setting $l^i=t^i_1$ eliminates the current action of $i$, as the conflict has been detected. 
Let $r^i=t^j_2$, since $j$ occupies $v$ in $(t^j_1, t^j_2]$ and $i$ occupies $v$ in $(t^j_2, t^j_2+\intv(\trans^i)]$ 
(Def.~\ref{def:DO}).
Similar reasoning applies to an \outin conflict.
Constraints for both \inin and \outin conflicts are denoted as Eq.~\ref{eq:cstrm}.

In a \waitin conflict $\cftij_W$, we simply add constraints to forbid agents to occupy $v$ at a time point $t_r=\min(t^i_2, t^j_2)$.
All constraints are denoted as Eq.~\ref{eq:cstrw}.
\begin{equation}\label{eq:cstrm}
    \begin{cases}
			\cstrm{i}{v^i_1}{v^i_2}{[t^i_1,t^j_2)}\\
			\cstrm{j}{v^j_1}{v^j_2}{[t^j_1,t^i_2)}\\
    \end{cases}
\end{equation}
\begin{equation}\label{eq:cstrw}
   \begin{cases}
		\cstro{i}{v}{t_r} \\
		\cstro{j}{v}{t_r} \\
   \end{cases} 
\end{equation}

\begin{remark}\label{csa-md}
The difference between \abbrSMWTP and \abbrCCBS lies in the constraints on \wait actions. \abbrCCBS adds a constraint to a specific \wait action, while \abbrSMWTP adds a constraint to a vertex $v$, which forbids all actions that occupy $v$.
The constraints in \abbrSMWTP avoid the problem about infinitely many wait actions.
\end{remark}

\subsection{Constraints on Multiple Actions (\abbrMMWTI)}\label{sec:mmwti}
For a specific vertex $v$, $e_1=(u_1,v) \in E$ and $e_2=(u_2,v) \in E$, the travel time for agent $i$, $\tau^i(u_1,v)$, can be different from $\tau^i(u_2,v)$.
Let $\durin^i(v),\durout^i(v)$ denote the minimum travel time of $i$'s \movein and \moveout at $v$, i.e., $\durin^i(v) = \min_{e=(u,v) \in E} (\tau^i(u,v))$ and $\durout^i(v) = \min_{e=(v,u) \in E} (\tau^i(v,u))$.
If agent $i$ starts to perform an \movein in $v$ at $t$, it occupies $v$ at least $(t, t+\durin^i(v)+\durout^i(v))$.
If agent $i$ starts to perform a \wait or \moveout in $v$ at $t$, it occupies $v$ at least $(t-\durin^i(v), t+\durout^i(v))$. 
\abbrMMWTI uses such DO to propagate constraints to multiple actions and time intervals.

\paragraph{Resolve IN($j$)-IN($i$)}
To permit $j$'s action, we add a constraint on $i$ to forbid all \movein actions starting within the time interval $[t^i_1, t^j_1+\durin^j+\durout^j)$, 
where $t^j_1+\durin^j+\durout^j$ represents the earliest time point for $j$ to leave $v$ if starting to perform an \movein in $v$ at $t^j_1$.
The same strategy is applicable for $i$:
\begin{equation}
    \begin{cases}\label{eq:cstr-inin}
			\cstrm{i}{*}{v^i_2}{[t^i_1, t^j_1+\durin^j+\durout^j)} \\
			\cstrm{j}{*}{v^j_2}{[t^j_1, t^i_1+\durin^i+\durout^i)} \\
    \end{cases} 
\end{equation}

\paragraph{Resolve OUT($j$)-IN($i$)}
To permit $j$'s action, we add a constraint on $i$ to forbid all \movein actions starting within the time interval $[t^i_1, t^j_1+\durout^j)$. To permit $i$'s action, multiple constraints on $j$ are added to forbid all \moveout actions 
starting within the time interval $T=[t^j_1, t^i_1+\durin^{i}+\durout^{i}+\durin^{j})$, 
and forbid \wait actions that occupy $v$ at any time point in $T$.
\begin{equation}\label{eq:cstr-outin-i}
	\cstrm{i}{*}{v}{[t^i_1, t^j_1+\durout^j)}
\end{equation}
\begin{equation}\label{eq:cstrs-outin-j}
    \begin{array}{ccc}
			\cstrw{j}{v}{[t^j_1, t^i_1+\durin^{i}+\durout^{i}+\durin^{j})} & \cup  \\
			\cstrm{j}{v}{*}{[t^j_1, t^i_1+\durin^{i}+\durout^{i}+\durin^{j})} &
    \end{array}
\end{equation}

Here, $[t^i_1, t^j_1+\durout^j)$ allows $i$ to either start the \movein earlier than $t^i_1$, or after the earliest time when $j$'s \moveout finishes ($\geq t^j_1+\durout^j$). $t^i_1+\durin^{i}+\durout^{i}$ is the earliest time point that $i$ finishes an \moveout action to leave $v$, and
$t^i_1+\durin^{i}+\durout^{i}+\durin^{j}$ is the earliest time point that $j$ finishes the \movein to reach $v$, so that $j$ can start to perform a \wait or \moveout afterwards.

\paragraph{Resolve WAIT($j$)-IN($i$)}
Due to the existence of \wait action and the variable duration of \wait action, this case requires extra care on the time interval of constraints. 
We first show the form of constraints, 
then we discuss how to properly set the time interval of constraints.

Overall, we add a constraint on $i$ to forbid all \movein actions starting within the time interval $T^i=[l^i, r^i)$,
or to forbid $j$'s \wait to occupy $v$ for at any time point in $T^j=[l^j, r^j)$,
where $T^i,T^j$ depend on the duration of $j$'s \wait action.

Similar to Eq. \ref{eq:cstrs-outin-j}, to permit $i$'s action, $t^i_1+\durin^{i}+\durout^{i}+\durin^{j}$ is a critical time point. 
$t^i_1+\durin^{i}+\durout^{i}$ is the earliest time point at which 
agent $i$ leaves the conflict vertex $v$.
Then agent $j$ can start the \movein action, arriving at $t^i_1+\durin^{i}+\durout^{i}+\durin^{j}$,
and performs a \wait.
Therefore, $t^i_1+\durin^{i}+\durout^{i}+\durin^{j}$ is the earliest time point that $j$ can start to perform a \wait without conflicting with $i$.
We set $r^j=t^i_1+\durin^{i}+\durout^{i}+\durin^{j}$.
To permit $j$'s action, $t^j_2+\durout^{j}$ is a critical time point. $t^j_2+\durout^{j}$ is the earliest time point at which agent $j$ leaves the conflict vertex $v$. Therefore, $t^j_2+\durout^{j}$ is the earliest time point that $i$ can start to perform a \movein without conflicting with $j$.
$r^i$ should be equal to $t^j_2+\durout^{j}$.
As illustrated in Fig.~\ref{fig:lowlevelsipp}, we can set $l^i = t^i_1$ and $l^j \in [t^j_1,t^j_2]$ to make the constraints affect the relevant actions.
To resolve conflicts efficiently without eliminating potential solutions and impacting completeness, we set $l^j = t^j_2$.

Now, we have $l^j = t^j_2$, $r^j=t^i_1+\durin^{i}+\durout^{i}+\durin^{j}$, $l^i = t^i_1$ and $r^i=t^j_2+\durout^{j}$, where $t^i_1$ is the start time of the \movein action and $t^j_2$ is the end time of the \wait action.
However, due to the uncertain duration of \wait action, the end time point $t^j_2$ may be very large, which can result in $l^j > r^j$ and thus invalidate the interval $T^j$. When $t^j_2$ is large, special care is needed.

If $t^j_2 \geq t^i_1+\durin^{i}+\durout^{i}+\durin^{j}$, $j$ performs a long \wait. 
The long wait action is divided into two consecutive short \wait actions.
The start and end times of the first short \wait action are ${t^{j}_{1}}' = t^j_1$ and ${t^{j}_{2}}' = t^i_1+\durin^{i}+\durout^{i}$ respectively.
The start and end times of the second short \wait action are ${t^{j}_{1}}'' =  t^i_1+\durin^{i}+\durout^{i}$ and ${t^{j}_{2}}'' = t^j_2$ respectively. 
We first resolve the conflict between $j$'s first short \wait action and $i$'s \movein action, using the above idea. 
Therefore,  we have $l^j = {t^{j}_{2}}'$, $r^j=t^i_1+\durin^{i}+\durout^{i}+\durin^{j}$, $l^i = t^i_1$ and $r^i={t^{j}_{2}}'+\durout^{j}$, where $t^i_1$ is the start time of the \movein action and ${t^{j}_{2}}'$ is the end time of the first short \wait action.

\begin{figure}[tb]
    \centering
    \includegraphics[width=0.9\linewidth]{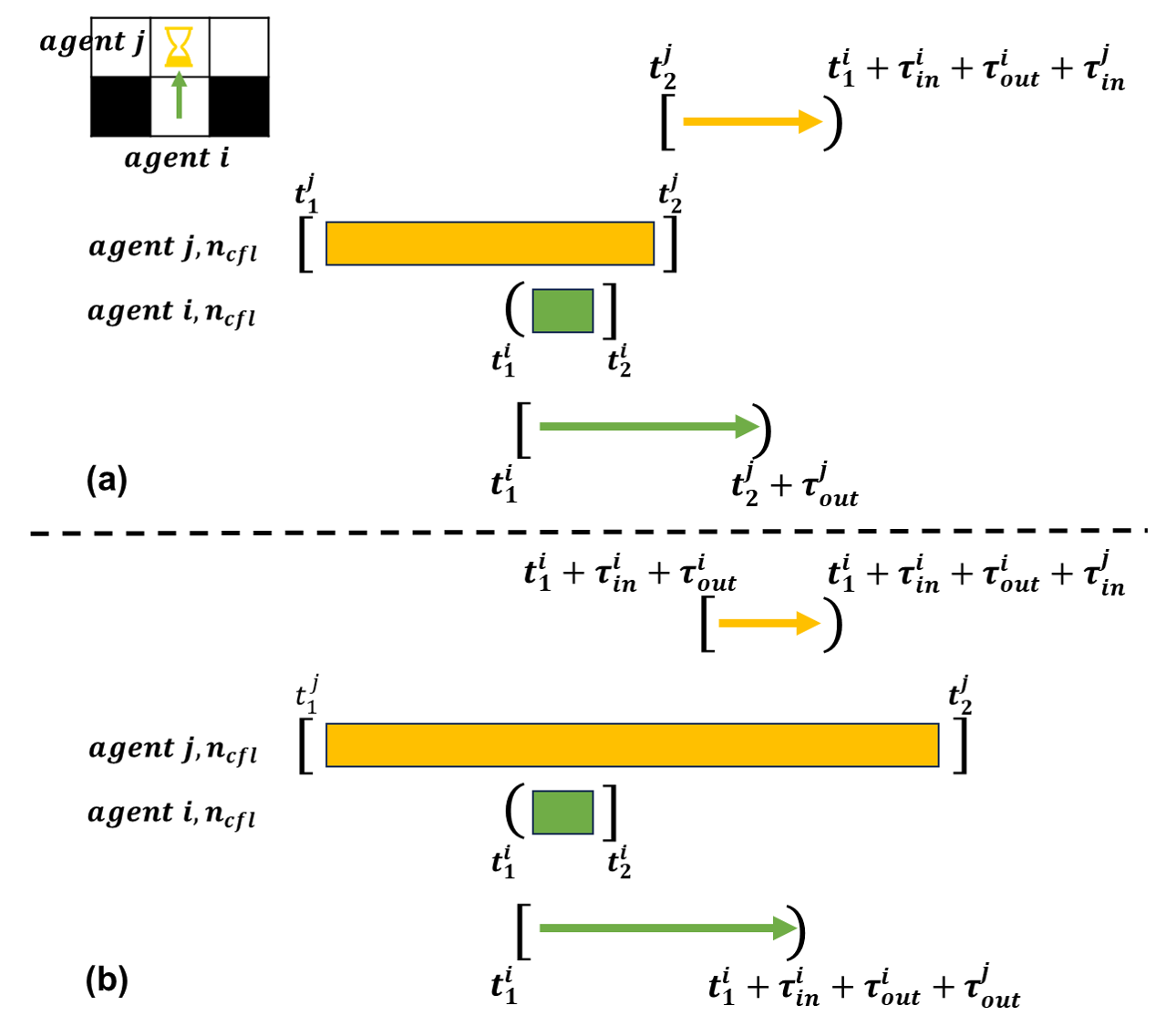}
    \caption{Resolve WAIT($j$)-IN($i$). 
    Rectangular strips are conflicts.
    Arrows are corresponding constraints.
    (a) Case 1:	$t^j_2<r^j$.
    (b) Case 2: $t^j_2\geq r^j$.
    }
    \label{fig:waitin}
\end{figure}
 
In summary, there are two cases:


Case 1 (Fig.~\ref{fig:waitin} (a)): $t^j_2<r^j$ 
\begin{equation}\label{eq:waitin-1}
   \begin{cases}
			\cstrm{i}{*}{v}{[t^i_1,t^j_2+\durout^{j})} \\
			\cstrw{j}{v}{[t^j_2,t^i_1+\durin^{i}+\durout^{i}+\durin^{j})}
   \end{cases} 
\end{equation}

Case 2 (Fig.~\ref{fig:waitin} (b)): $t^j_2\geq r^j$ 
		\begin{equation}\label{eq:waitin-2}
			\begin{cases}
			\cstrm{i}{*}{v}{[t^i_1, t^i_1+\durin^{i}+\durout^{i}+\durout^{j})} \\
			\cstrw{j}{v}{[t^i_1+\durin^{i}+\durout^{i}, t^i_1+\durin^{i}+\durout^{i}+\durin^{j})}
			\end{cases}
		\end{equation}
        
Note that conflict resolution in Case 2 needs multiple iterations to permit $j$'s long \wait action. In the first iteration, after adding the constraint of case 2, the start time of $i$'s \movein action is delayed to permit $j$'s first short \wait action.
Then in the second iteration, we have a conflict between $j$'s second short \wait action and $i$'s delayed \movein action. After a finite number of iterations, case 2 will return to case 1, and finally the entire long \wait action of $j$ is permitted.

\begin{figure}[tb]
    \centering
    \includegraphics[width=0.8\linewidth]{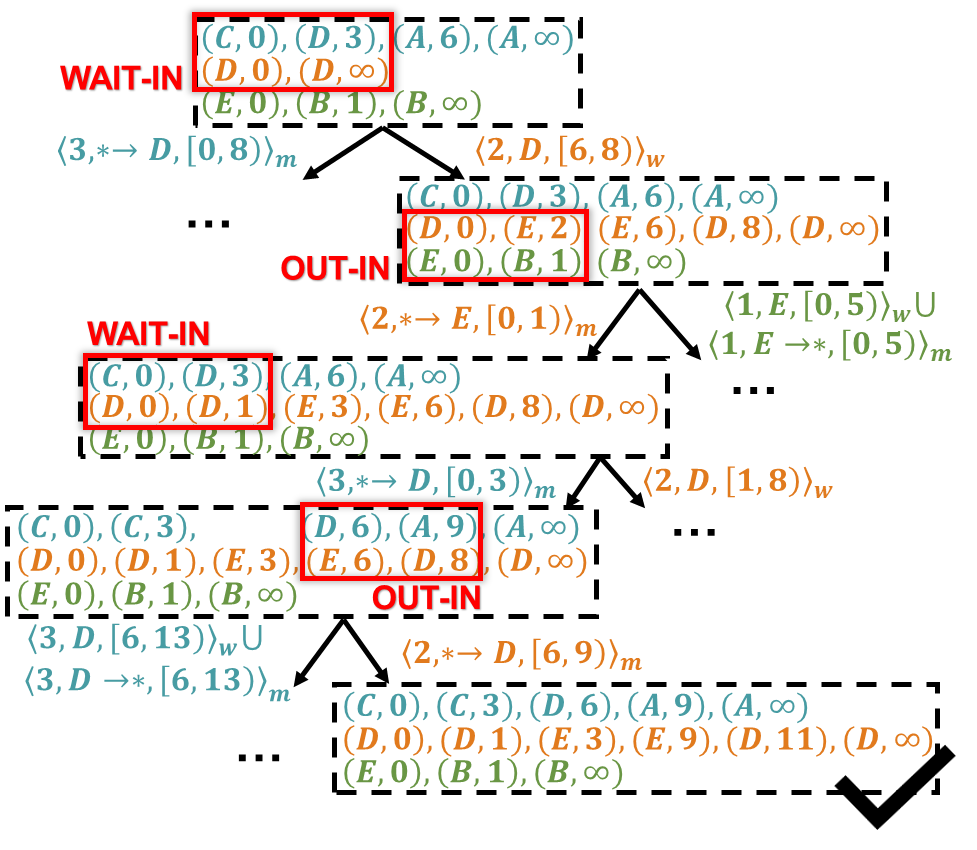}
    \caption{Search process of \abbrMMWTI for Ex.~\ref{ex:toyexample_ccbs}
    }
    \label{fig:toyexample_CMA}
\end{figure}


\begin{example}
Following Ex.~\ref{ex:toyexample_ccbs}, Fig.~\ref{fig:toyexample_CMA} shows the search process of \abbrMMWTI. 
\abbrMMWTI can eventually find an optimal solution.
Four key conflicts need to be resolved.
The first conflict is $\cftij_W$, where  $i = 3$, $j = 2$, $\trans^i=((C,0), (D, 3))$ and $\trans^j=((D,0), (D, \infty))$.
The second conflict is $\cftij_O$, where  $i = 2$, $j = 1$, $\trans^i=((D,0), (E, 2))$ and $\trans^j=((E,0), (B, 1))$.
The third conflict is $\cftij_W$, where  $i = 3$, $j = 2$, $\trans^i=((C,0), (D, 3))$ and $\trans^j=((D,0), (D, 1))$.
The fourth conflict is $\cftij_O$, where  $i = 2$, $j = 3$, $\trans^i=((E,6), (D, 8))$ and $\trans^j=((D,6), (A, 9))$.
\end{example}

\subsection{Discussion on \abbrMMWTI and \abbrSMWTP}
For each constraint generated in \abbrMMWTI and \abbrSMWTP, the end time $r$ of the time interval in that constraint must be no smaller than the start time $l$ of the time interval.
Otherwise, the interval and the constraint are invalid.
It is easy to verify that $r$ is not smaller than $l$ in the constraints of \abbrSMWTP.
But Eq.~\ref{eq:cstr-inin}, Eq.~\ref{eq:cstr-outin-i} and Eq.~\ref{eq:cstrs-outin-j} of \abbrMMWTI may result in $r$ being smaller than $l$ due to the differences in travel time.
\begin{example}
As shown in Fig.~\ref{fig:assumptionExample} (a), the conflict is $\cftij_I$, where  $i = 1$, $j = 2$, $\trans^i=((A,5), (B, 11))$ and $\trans^j=((D,10.5), (B, 11.5))$.
According to Eq.~\ref{eq:cstr-inin}, the constraint added to agent $j$ is $\cstrm{j}{*}{B}{[10.5, 9)}$, which is invalid.

\end{example}

\begin{figure}[tb]
    \centering
    \includegraphics[width=0.8\linewidth]{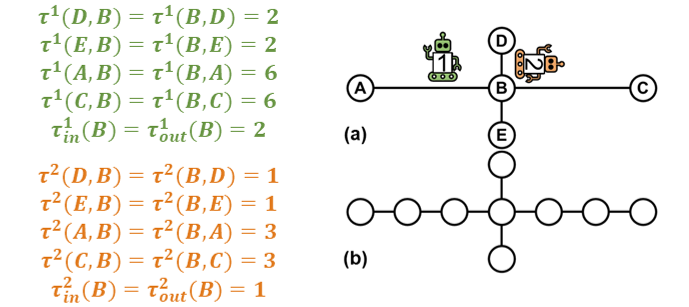}
    \caption{
    (a) An example where \abbrMMWTI generates an invalid constraint.
    (b) By inserting intermediate vertices into the longer edges, the constraint generated by \abbrMMWTI can be valid.
    }
    \label{fig:assumptionExample}
\end{figure}

If the following assumption holds, then $r$ is always greater than $l$ for any generated constraint in \abbrMMWTI.
\begin{assumption}\label{assume:equal_time}
    For each agent $i\in I$, the duration for $i$ to traverse any edge $e\in E$ is the same constant real number. This constant can be different for different agents.
\end{assumption}
In practice, Assumption~\ref{assume:equal_time} can be satisfied by inserting intermediate vertexes into the longer edges (as shown in Fig.~\ref{fig:assumptionExample} (b)) and Keeping agents moving at constant speeds.
The grid world used for path planning in an automated warehouse is a common example.
By Assumption~\ref{assume:equal_time}, all edges linked to $v$ have the same travel time for agent $i$. 
The time intervals in Eq.~\ref{eq:cstr-inin} are changed to $[t^i_1,t^j_2+\tau^j)$ and $[t^j_1,t^i_2+\tau^{i})$ where $\tau^i$, $\tau^j$ denote the travel time of agent $i$, $j$. 
It is easy to verify that $r$ is not smaller than $l$ in these new time intervals.
The same is true for Eq.~\ref{eq:cstr-outin-i} and Eq.~\ref{eq:cstrs-outin-j}.

%% file: analysis_v2.tex
\subsection{Analysis}

In this subsection, we show that CBS-AA is optimal and complete. 
We begin by showing that the constraints introduced by \abbrCBSAA are mutually disjunctive, a property essential for proving optimality. 
Next, we demonstrate that \abbrCBSAA always terminates within a finite number of steps.
Finally, we prove the optimality and completeness of \abbrCBSAA.

\begin{definition}\label{def:MD}
Two constraints are Mutually Disjunctive (MD)~\cite{li2019multi} if a set of conflict-free paths cannot simultaneously violate them.
In other words, if a set of paths simultaneously violates these two constraints, then it must have a conflict.
\end{definition}

\begin{lemma}\label{lemma:csa}
Constraints in \abbrSMWTP (Eq.~\eqref{eq:cstrm}, ~\eqref{eq:cstrw}) are MD.
\end{lemma}
\begin{proof}\label{csa-md}
In Eq~\eqref{eq:cstrm}, the end time $r$ of the time interval added to $i$ ($j$) is $\Tt{\trans^j}$ ($\Tt{\trans^i}$).
Therefore, simultaneously violating these two constraints necessarily leads to a conflict.
In Eq~\eqref{eq:cstrw}, violating both constraints will obviously result in a conflict at time point $t_r$. 
\end{proof}

Assumption~\ref{assume:equal_time} allows us to use $\tau^i,\tau^j$ to denote the travel time of agent $i,j\in I$ for any edge in $G$, which greatly simplifies the notation. This section thus uses this assumption to show the ideas in the proofs.
The proof can be extended to the general case without relying on Assumption~\ref{assume:equal_time} by using $\tau^i_{in}(v),\tau^i_{out}(v)$.
Given $\cftij$, let $\trans^i=((v^i_1,t^i_1),(v^i_2,t^i_2))$ and $\trans^j=((v^j_1, t^j_1),(v^j_2,t^j_2))$.

\begin{lemma}\label{lemma:in-in}
\inin constraints (Eq.~\eqref{eq:cstr-inin}) in \abbrMMWTI are MD.
\end{lemma}

\begin{proof}
    Assume agent $i$ performs \movein at time $x$, occupying $v$ over $\intv^i_x=(x, x+2\tau^i)$, and agent $j$ performs \movein at time $y$, occupying $v$ over $\intv^j_y=(y, y+2\tau^j)$.
    If both agents violate the constraint, then
    $x \in [t_1^i, t_2^j+\tau^j)$ and
    $y \in [t_1^j, t_2^i+\tau^i)$.
    We show the intervals must intersect. By contradiction, suppose they do not.
        (i) If $x \ge y+2\tau^j$, then
        $x < t_2^j+\tau^j$ while
        $y+2\tau^j \ge t_1^j+2\tau^j = t_2^j+\tau^j$,
        a contradiction.
        (ii) If $y \ge x+2\tau^i$, then
        $y < t_2^i+\tau^i$ while
        $x+2\tau^i \ge t_1^i+2\tau^i = t_2^i+\tau^i$,
        a contradiction.
    Hence, the intervals intersect, implying a conflict.
    Without Assumption~\ref{assume:equal_time}, replacing $2\tau$ by $\tau_{in}+\tau_{out}$ yields the same conclusion. Therefore, constraints in Eq.~\eqref{eq:cstr-inin} are MD.
\end{proof}

With the same idea, it can be proved that \outin constraints (Eq.~\eqref{eq:cstr-outin-i},~\eqref{eq:cstrs-outin-j}) and \waitin constraints (Eq.~\eqref{eq:waitin-1} and~\eqref{eq:waitin-2}) are MD.

\begin{lemma}\label{lemma:termination}
\abbrCBSAA can terminate within a finite number of steps on a solvable \abbrMAPFAA problem.
\end{lemma}

\begin{proof}

    Due to space limits, we sketch the main idea, similar to the one in~\cite{combrink2025sound}.
If the method does not terminate, there exists an infinite sequence of high-level nodes whose costs are all below the optimal solution and all contain conflicts.

Let a trajectory $\sigma$ be a finite sequence of move actions with cost $g(\sigma)$ equal to the sum of their durations. If a path $\pi$ contains all actions of $\sigma$ in order, we write $\pi \sim \sigma$. For any finite $c\in\mathbb{R}$, the set $\{\sigma \mid g(\sigma)<c\}$ is finite. Hence, infinitely many constraints must be added to paths mapping to a specific trajectory $\sigma^{\infty}$.

Non-termination requires that, despite infinitely many constraints, there still exists a path $\pi' \sim \sigma^{\infty}$ satisfying all constraints with $g(\pi') < c$. However, each constraint introduced by \abbrCBSAA reduces the feasible execution time of actions in $\sigma^{\infty}$ by a non-zero amount. After finitely many branchings, no such $\pi'$ can exist. Therefore, \abbrCBSAA must terminate in finite steps.
\end{proof}


\begin{theorem}
\label{theorem:OptimalAndComplete}
\abbrCBSAA is optimal and solution complete.
\end{theorem}

\begin{proof}
    Because all constraints added by \abbrCBSAA are MD (Lem.~\ref{lemma:csa} and Lem.~\ref{lemma:in-in}), all solutions are reachable from the root of search tree.
    And \abbrCBSAA always expands the high-level node with minimum objective value.
    Thus, \abbrCBSAA is optimal. 
    If a feasible solution exists, \abbrCBSAA is solution complete because it can terminate within a finite number of steps (Lem.~\ref{lemma:termination}).
\end{proof}

%% file: result.tex
\begin{figure*}[htbp] 
    \centering
    \includegraphics[width=\textwidth]{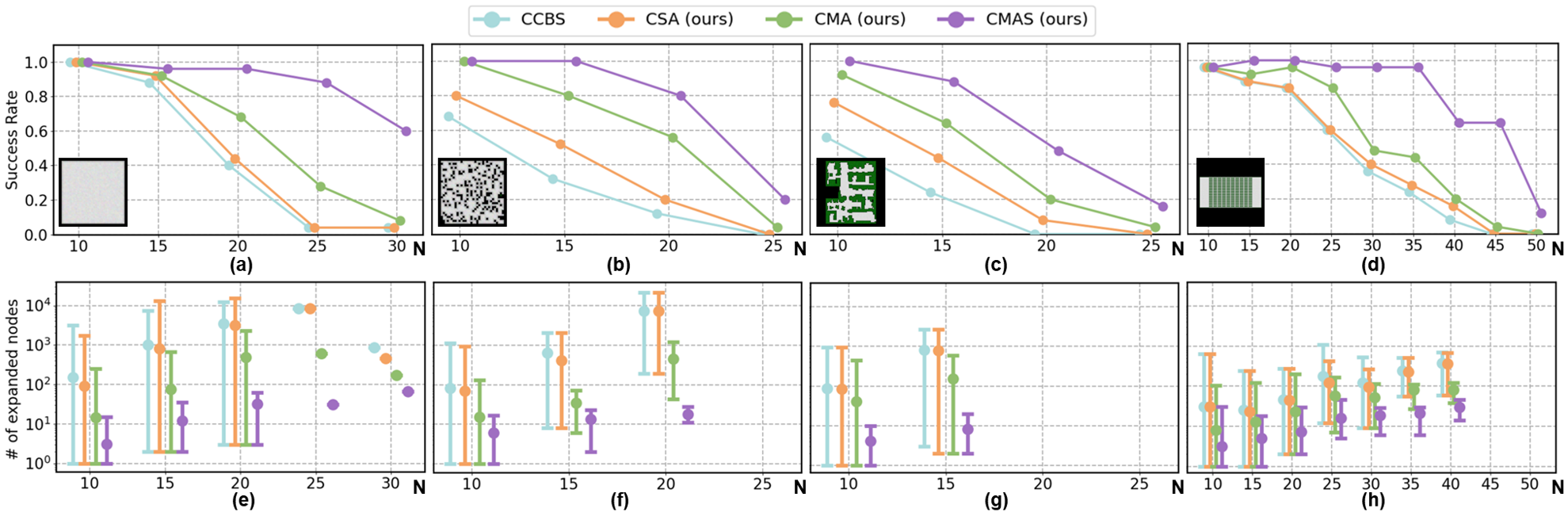} 
    \caption{(a)-(d): Success rates of the algorithms. $N$ is the number of agents.
    (e)-(h): Min., Avg., and Max. number of high-level nodes expanded by the algorithms, which only counts the cases where all four methods succeed. $N$ is the number of agents.
    }
    \label{fig:resultall}
\end{figure*}



In the previous section, we proposed \abbrSMWTP to correct the errors of \abbrCCBS, and further proposed \abbrMMWTI to resolve conflicts efficiently. In order to find conflict-free solutions faster, we proposed a new low-level planner, \abbrSIPPSWC.
In this section, we compare our \abbrSMWTP, \abbrMMWTI and \abbrMMWTIS (\abbrMMWTI with \abbrSIPPSWC) with \abbrCCBS~\cite{ANDREYCHUK2022103662}, to show the gradual improvement of the methods.
Among them, \abbrSMWTP and \abbrMMWTI use \abbrSIPP at the low-level, \abbrMMWTIS uses \abbrSIPPSWC at the low-level.
Besides, to verify the optimality, we compare \abbrMMWTIS with \abbrLSS~\cite{ren2021loosely}.



We use four maps of different sizes from an online data set~\cite{stern2019multi}: ``empty-32-32'', ``random-32-32-20'', ``den312d'' and ``warehouse-10-20-10-2-2''.
We test the algorithms by varying the number of agents $N$ from 10 to 50.
For testing purposes, we consider that the edges in each map have fixed-length edges of a unit, and each agent has a random speed between 1 and 20.
In other words, it takes the same amount of time to traverse any edge for the same agent, and it takes different amount of time to traverse the same edge for different agents. 
When agents reach their goals, they will stay there permanently.
The runtime limit of each instance is 30 seconds.
We implement all algorithms in C++ and run all tests on a computer with an Intel Core i7-11800H CPU.

\subsection{Comparison with \abbrCCBS}
\paragraph{Success Rates} Fig.~\ref{fig:resultall} (a)-(d) show the success rates of the methods.
We observe that \abbrSMWTP has a slightly higher success rate than \abbrCCBS, mainly due to the changes in the wait action.
Different from \abbrCCBS, the constraint added by \abbrSMWTP for wait actions (Fig.~\ref{fig:lowlevelsipp} (c)) may affect all of \movein, \moveout, \wait actions, and therefore resolve conflicts more efficiently and improve the success rates a bit.
Since \abbrCCBS's constraints on wait actions are not mutually disjunctive, it is possible that \abbrCCBS overlooks feasible solutions and fails within the time limit, resulting in a lower success rate.
Additionally, the success rates of \abbrMMWTI are obviously higher than \abbrSMWTP and \abbrCCBS, especially when the number of agents increases.
The reason is that, \abbrMMWTI add constraints that affect multiple actions and add constraints at a time interval for \wait actions, which makes \abbrMMWTI resolve conflicts more efficiently than \abbrSMWTP and \abbrCCBS.
Finally, \abbrSIPPSWC can help find an optimal path and avoid collisions with other agents as much as possible, especially for cases involving waiting at a safe interval. 
With the help of \abbrSIPPSWC, the success rates of \abbrMMWTIS are further improved and can handle up to 50 agents in the warehouse map, which is otherwise hard to achieve by the other three algorithms.



\paragraph{Number of Expansions} Fig.~\ref{fig:resultall} (e)-(f) shows the number of high-level nodes expanded.
\abbrMMWTI has the least number of expansions, which means it can resolve more conflicts with one branching, which reduces the number of expansions to find a solution. 
Compared with \abbrSMWTP, \abbrMMWTI reduces the number of expansions by up to 90\%: When the number of agents is 25 in the empty map, the average number of expansions in \abbrSMWTP is 8286 while the one for \abbrMMWTI is 617.
By considering paths of other agents in the low-level planner, \abbrMMWTIS can find an optimal solution with even fewer expansions, which thus enhances the success rates.





\begin{figure}[tb]
    \centering
    \includegraphics[width=0.9\linewidth]{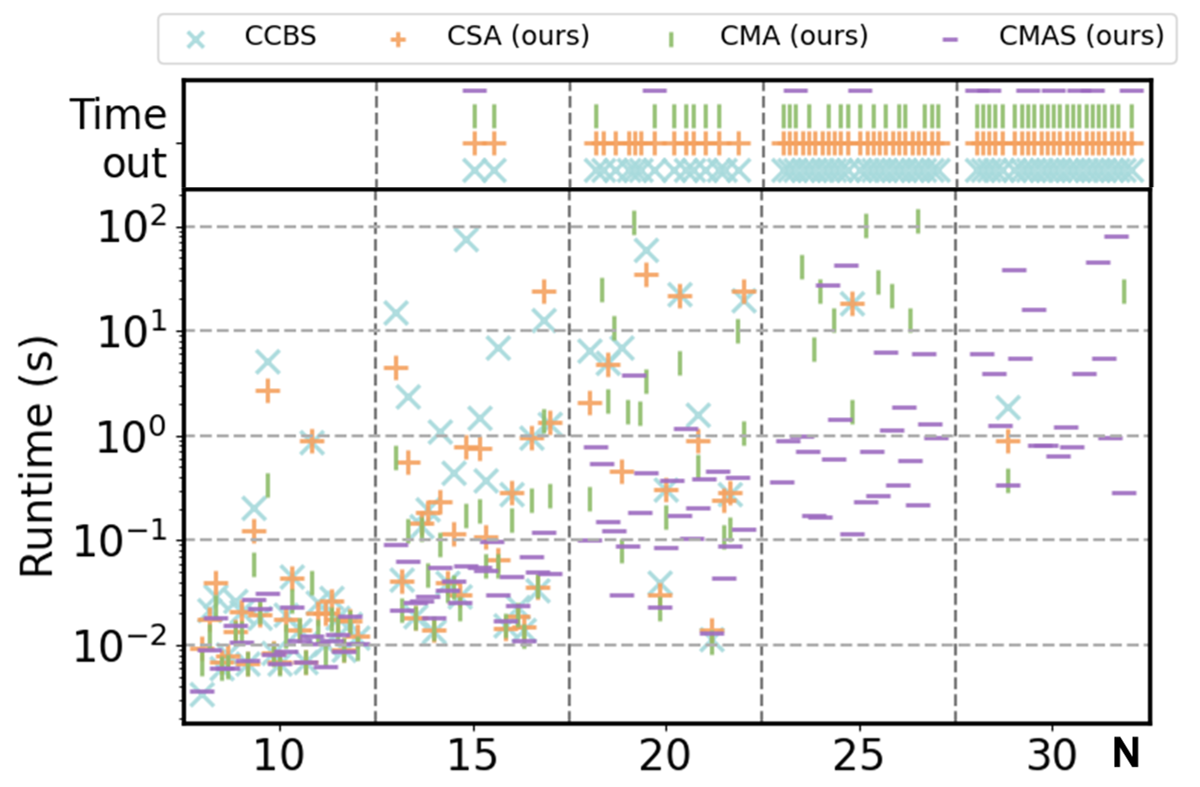}
    \caption{Runtime in empty $32\times32$. $N$ is the number of agents.
    }
    \label{fig:runtime}
\end{figure}

\paragraph{Runtime} We further extend the runtime limit to 120 seconds and analyze the runtime of different algorithms in empty $32\times32$. As shown in Fig.~\ref{fig:runtime}, when the number of agents $N \geq 25$, the success rates of \abbrCCBS and \abbrSMWTP are low, while the success rate of \abbrMMWTI drops rapidly. However, the success rate of \abbrMMWTIS remains above 60\%. When the number of agents $N = 30$, \abbrMMWTIS can find solutions in 68\% of instances within 60 seconds.

\paragraph{\abbrSIPP v.s. \abbrSIPPSWC} We compare the average runtime per call of \abbrSIPP in \abbrMMWTI with that of \abbrSIPPSWC in \abbrMMWTIS in empty $32\times32$. As shown in Table.~\ref{tab:lowlevel}, the runtime of \abbrSIPPSWC increases with the number of agents. The increase in the number of agents leads to an increase in the number of soft constraints contained in \abbrSIPPSWC. Multiple child nodes can be generated for the same safe interval (depending on the number of soft constraints involved), thus increasing the number of nodes to explore. However, due to the consideration of other agents' paths, \abbrMMWTIS can find conflict-free solutions faster.



\begin{table}[tb]
\centering
\caption{Comparison of \abbrSIPP with \abbrSIPPSWC in terms of average runtime per call (ms).} 
\label{tab:lowlevel} 
\small 
\begin{tabular}{|c|c|c|c|c|c|}
\hline
N & 10 & 15 & 20 & 25 & 30 \\ \hline
\abbrSIPP & 0.803 & 0.930 & 0.902 & 0.999 & 1.03 \\ \hline 
\abbrSIPPSWC & 0.845 & 1.20 & 1.37 & 1.62 & 1.95 \\ \hline 
\end{tabular}

\end{table}

\subsection{Comparison with \abbrLSS}

\begin{table}[tb]
\centering
\caption{Comparison of \abbrMMWTIS with \abbrLSS in terms of Success Rate (SR) and Avg. Run Time (RT) under different Time Limit (TL). 
The data are shown in format (\abbrMMWTIS / \abbrLSS under $TL=30s$ / \abbrLSS under $TL=120s$).} 
\label{tab:lss} 
\resizebox{\linewidth}{!}{
\begin{tabular}{|c|c|c|c|c|c|}
\hline
N & 2 & 4 & 6 & 8  \\ \hline
SR & 1.0/1.0/1.0 & 1.0/0.96/1.0 & 1.0/0.2/0.4 & 1.0/0.0/0.0\\ \hline 
RT (s) & 0.009/0.895/0.895 & 0.002/1.52/4.44 & 0.082/19.0/49.6 & -/-  \\ \hline 
\end{tabular}
}
\end{table}

\abbrLSS~\cite{ren2021loosely} is an A*-based \abbrMAPF planner and can find an optimal solution in \abbrMAPFAA.
We compare \abbrMMWTIS with it in order to verify the optimality of \abbrMMWTIS.
The experimental setup is the same as above. 
We fix the map to random-32-32-20 and the number of agents is $N = {2,4,6,8}$.
And considering that \abbrLSS needs a long time to search for an optimal solution, we add an additional experiment to extend the runtime limit to 120 seconds.
As shown in Table.~\ref{tab:lss}, when the number of agents increases to more than 4, the success rate of \abbrLSS begins to decline.
When the number of agents reaches 8, the success rate of \abbrLSS is 0\%.
However, \abbrMMWTIS can always maintain a 100\% success rate, and handle more agents than \abbrLSS.
The cost of the solution found by \abbrMMWTIS is the same as \abbrLSS, while \abbrMMWTIS can find the solution in less run time.
The results show that \abbrMMWTIS can find the same optimal solution as \abbrLSS in a shorter time.

%% file: conclude.tex




This paper focuses on \abbrMAPFAA, and develops a new exact algorithm \abbrCBSAA for \abbrMAPFAA with solution optimality guarantees, based on the popular CBS framework.
\abbrCBSAA introduces new conflict resolution techniques for agents with asynchronous actions, which improves the runtime efficiency of the algorithm.
Experimental results demonstrate the advantages of our new approaches in different settings against several baseline methods.

For future work, one can also consider speed and uncertainty in \abbrMAPFAA or combine \abbrMAPFAA with target allocation and sequencing~\cite{ren23cbssTRO,2025_IROS_MCPFTT_XuemianWu}.

%% file: sample.bib
@inproceedings{2025_IROS_MCPFTT_XuemianWu,
  author = {Wu, Xuemian and Ren, Zhongqiang},
  booktitle = {2025 IEEE/RSJ International Conference on Intelligent Robots and Systems (IROS)},
  title = {Multi-Agent Combinatorial Path Finding for Tractor-Trailers in Occupancy Grids},
  year = {2025},
  volume = {},
  number = {},
  pages = {14124-14131},
  url = {https://rap-lab.github.io/documents/publications/2025_IROS_MCPFTT_XuemianWu.pdf},
  doi = {10.1109/IROS60139.2025.11246219}
}

@article{ren23cbssTRO,
  author = {Ren, Zhongqiang and Rathinam, Sivakumar and Choset, Howie},
  journal = {IEEE Transactions on Robotics},
  title = {{CBSS}: A New Approach for Multiagent Combinatorial Path Finding},
  year = {2023},
  volume = {39},
  number = {4},
  pages = {2669-2683},
  doi = {10.1109/TRO.2023.3266993},
  url = {https://rap-lab.github.io/documents/publications/ren23_CBSS_TRO.pdf},
  code = {https://github.com/rap-lab-org/public_pymcpf}
}

@inproceedings{2025_SOCS_LSRPstar_ShuaiZhou,
  title = {LSRP*: Scalable and Anytime Planning for Multi-Agent Path Finding with Asynchronous Actions (Extended Abstract)},
  author = {Zhou, Shuai and Zhao, Shizhe and Ren, Zhongqiang},
  booktitle = {Proceedings of the International Symposium on Combinatorial Search},
  volume = {18},
  number = {1},
  year = {2025},
  month = jul,
  pages = {275-276},
  doi = {10.1609/socs.v18i1.36016},
  url = {https://rap-lab.github.io/documents/publications/2025_SOCS_LSRPstar_ShuaiZhou.pdf}
}

@article{li2025cbs,
  title={Revisiting Conflict Based Search with Continuous-Time},
  author={Li, Andy and Chen, Zhe and Harabor, Danial and Vered, Mor},
  journal={arXiv preprint arXiv:2501.07744},
  year={2025}
}

@inproceedings{phillips2011sipp,
	title={Sipp: Safe interval path planning for dynamic environments},
	author={Phillips, Mike and Likhachev, Maxim},
	booktitle={2011 IEEE International Conference on Robotics and Automation},
	pages={5628--5635},
	year={2011},
	organization={IEEE}
}

@inproceedings{ren2021loosely,
  title={Loosely synchronized search for multi-agent path finding with asynchronous actions},
  author={Ren, Zhongqiang and Rathinam, Sivakumar and Choset, Howie},
  booktitle={2021 IEEE/RSJ International Conference on Intelligent Robots and Systems (IROS)},
  pages={9714--9719},
  year={2021},
  organization={IEEE}
}

@inproceedings{stern2019multi,
  title={Multi-agent pathfinding: Definitions, variants, and benchmarks},
  author={Stern, Roni and Sturtevant, Nathan and Felner, Ariel and Koenig, Sven and Ma, Hang and Walker, Thayne and Li, Jiaoyang and Atzmon, Dor and Cohen, Liron and Kumar, TK and others},
  booktitle={Proceedings of the International Symposium on Combinatorial Search},
  volume={10},
  number={1},
  pages={151--158},
  year={2019}
}

@article{okumura2022priority,
  title = {Priority Inheritance with Backtracking for Iterative Multi-agent Path Finding},
  journal = {Artificial Intelligence},
  pages = {103752},
  year = {2022},
  issn = {0004-3702},
  doi = {https://doi.org/10.1016/j.artint.2022.103752},
  author = {Keisuke Okumura and Manao Machida and Xavier Défago and Yasumasa Tamura},
}

@article{ANDREYCHUK2022103662,
title = {Multi-agent pathfinding with continuous time},
journal = {Artificial Intelligence},
volume = {305},
pages = {103662},
year = {2022},
issn = {0004-3702},
doi = {https://doi.org/10.1016/j.artint.2022.103662},
url = {https://www.sciencedirect.com/science/article/pii/S0004370222000029},
author = {Anton Andreychuk and Konstantin Yakovlev and Pavel Surynek and Dor Atzmon and Roni Stern}
}

@article{wagner2015subdimensional,
  title={Subdimensional expansion for multirobot path planning},
  author={Wagner, Glenn and Choset, Howie},
  journal={Artificial intelligence},
  volume={219},
  pages={1--24},
  year={2015},
  publisher={Elsevier}
}

@article{sharon2015conflict,
  title={Conflict-based search for optimal multi-agent pathfinding},
  author={Sharon, Guni and Stern, Roni and Felner, Ariel and Sturtevant, Nathan R},
  journal={Artificial intelligence},
  volume={219},
  pages={40--66},
  year={2015},
  publisher={Elsevier}
}

@inproceedings{de2013push,
  title={Push and rotate: cooperative multi-agent path planning},
  author={De Wilde, Boris and Ter Mors, Adriaan W and Witteveen, Cees},
  booktitle={Proceedings of the 2013 international conference on Autonomous agents and multi-agent systems},
  pages={87--94},
  year={2013}
}

@inproceedings{yu2013structure,
  title={Structure and intractability of optimal multi-robot path planning on graphs},
  author={Yu, Jingjin and LaValle, Steven},
  booktitle={Proceedings of the AAAI Conference on Artificial Intelligence},
  volume={27},
  number={1},
  pages={1443--1449},
  year={2013}
}

@inproceedings{walker2018extended,
  title={Extended Increasing Cost Tree Search for Non-Unit Cost Domains.},
  author={Walker, Thayne T and Sturtevant, Nathan R and Felner, Ariel},
  booktitle={IJCAI},
  pages={534--540},
  year={2018}
}

@inproceedings{li2019multi,
  title={Multi-agent path finding for large agents},
  author={Li, Jiaoyang and Surynek, Pavel and Felner, Ariel and Ma, Hang and Kumar, TK Satish and Koenig, Sven},
  booktitle={Proceedings of the AAAI Conference on Artificial Intelligence},
  volume={33},
  number={01},
  pages={7627--7634},
  year={2019}
}

@inproceedings{barer2014suboptimal,
  title={Suboptimal variants of the conflict-based search algorithm for the multi-agent pathfinding problem},
  author={Barer, Max and Sharon, Guni and Stern, Roni and Felner, Ariel},
  booktitle={Proceedings of the international symposium on combinatorial Search},
  volume={5},
  number={1},
  pages={19--27},
  year={2014}
}

@inproceedings{li2021eecbs,
  title={Eecbs: A bounded-suboptimal search for multi-agent path finding},
  author={Li, Jiaoyang and Ruml, Wheeler and Koenig, Sven},
  booktitle={Proceedings of the AAAI conference on artificial intelligence},
  volume={35},
  number={14},
  pages={12353--12362},
  year={2021}
}

@article{okumura2021time,
  title={Time-Independent Planning for Multiple Moving Agents},
  volume={35},
  number={13},
  journal={Proceedings of the AAAI Conference on Artificial Intelligence},
  author={Okumura, Keisuke and Tamura, Yasumasa and Défago, Xavier},
  year={2021},
  month={May},
  pages={11299-11307}
}

@inproceedings{li2022mapf,
  title={MAPF-LNS2: Fast repairing for multi-agent path finding via large neighborhood search},
  author={Li, Jiaoyang and Chen, Zhe and Harabor, Daniel and Stuckey, Peter J and Koenig, Sven},
  booktitle={Proceedings of the AAAI Conference on Artificial Intelligence},
  volume={36},
  number={9},
  pages={10256--10265},
  year={2022}
}

@inproceedings{zhou2025loosely,
  title={Loosely Synchronized Rule-Based Planning for Multi-Agent Path Finding with Asynchronous Actions},
  author={Zhou, Shuai and Zhao, Shizhe and Ren, Zhongqiang},
  booktitle={Proceedings of the AAAI Conference on Artificial Intelligence},
  volume={39},
  number={14},
  pages={14763--14770},
  year={2025}
}

@article{combrink2025sound,
  title={Optimal Multi-agent Path Finding in Continuous Time},
  author={Combrink, Alvin and Roselli, Sabino Francesco and Fabian, Martin},
  journal={arXiv preprint arXiv:2508.16410},
  year={2025}
}
